\documentclass[preprint,12pt]{elsarticle}

\usepackage{amssymb}
\usepackage{amsmath}

\usepackage{latexsym}
\usepackage{amssymb}
\usepackage{amsmath}
\usepackage{amsthm}
\usepackage{booktabs}
\usepackage{enumitem}
\usepackage{graphicx}
\usepackage{color}
\usepackage{bm}
\usepackage{algorithm}
\usepackage{algorithmic}
\usepackage{hyperref}
\usepackage{subcaption}
\usepackage{multirow}

\newcommand{\x}{\mathbf{x}}
\newcommand{\y}{\mathbf{y}}

\newcommand{\g}{\mathbf{g}}

\newcommand{\m}{\mathbf{m}}
\renewcommand{\c}{\mathbf{c}}

\newcommand{\btheta}{\bm{\theta}}
\newcommand{\bzeta}{\bm{\zeta}}
\newcommand{\bdelta}{\bm{\delta}}
\newcommand{\bxi}{\bm{\xi}}

\newcommand{\E}{\mathbb{E}}
\newcommand{\bO}{{\cal O}}
\newcommand{\R}{\mathbb{R}}

\newcommand{\sign}{\mbox{sign}}
\newcommand{\<}{\left\langle}
\renewcommand{\>}{\right\rangle}

\newtheorem{theorem}{Theorem}
\newtheorem{lemma}[theorem]{Lemma}
\newtheorem{corollary}[theorem]{Corollary}

\newcounter{assumption}
\newcommand{\assumption}[1]{%
  \refstepcounter{assumption}\label{#1}%
  \textbf{Assumption \theassumption}%
}

\journal{arXiv}
\myfooter[L]{%
  \parbox[t]{\textwidth}{%
    The Chinese version has been accepted by the Chinese Journal of Computers.\\
    Preprint submitted to arXiv\hfill June 22, 2026%
  }%
}

\begin{document}

\begin{frontmatter}



\title{Convergence Rate Analysis of LION}

\author[label1]{Yiming Dong}\ead{yimingdong_ml@outlook.com}
\author[label2]{Huan Li}\ead{lihuanss@nankai.edu.cn}
\author[label1,label3]{Zhouchen Lin\corref{cor1}}\ead{zlin@pku.edu.cn}
\cortext[cor1]{Corresponding author at: State Key Lab of General Artificial Intelligence, School of Intelligence Science and Technology, Peking University, Beijing, 100871, China.}
\affiliation[label1]{organization={State Key Lab of General Artificial Intelligence, School of Intelligence Science and Technology, Peking University},
            city={Beijing},
            postcode={100871},
            country={China}}
\affiliation[label2]{organization={Institute of Robotics and Automatic Information Systems, College of Artificial Intelligence, Nankai University},
            city={Tianjin},
            postcode={300071},
            country={China}
            }
\affiliation[label3]{organization={Pazhou Laboratory (Huangpu)},
            city={Guangzhou},
            postcode={510555},
            country={China}}

\begin{abstract}
The LION (evoLved sIgn mOmeNtum) optimizer for deep neural network training was found by Google via program search, with the simple sign update yet showing impressive performance in training large scale networks. Although previous studies have investigated its convergence properties, a comprehensive analysis, especially the convergence rate, is still desirable. Recognizing that LION can be regarded as solving a specific constrained problem, this paper focuses on demonstrating its convergence to the Karush-Kuhn-Tucker (KKT) point at the rate of $\bO(\sqrt{d}K^{-1/4})$ measured by gradient \(\ell_1\) norm, where $d$ is the problem dimension and $K$ is the number of iteration steps. Step further, we remove the constraint and establish that LION converges to the critical point of the general unconstrained problem at the same rate. This rate not only delivers the currently optimal dependence on the problem dimension \(d\) but also tightly matches the theoretical lower bound for nonconvex stochastic optimization algorithms, which is typically measured using the gradient \(\ell_2\) norm, with respect to the number of iterations \(K\). Through extensive experiments, we not only demonstrate that LION achieves lower loss and higher performance compared to standard SGD, but also empirically confirm that the gradient \(\ell_1/\ell_2\) norm ratio aligns with \(\Theta(\sqrt{d})\), thus proving that our convergence rate matches the theoretical lower bound with respect to \(d\) in the empirical sense.
\end{abstract}




\end{frontmatter}



\section{Introduction}
In the realm of deep learning \cite{lecun2015deep}, the efficacy of model training critically hinges on the choice of optimization algorithms. The optimization algorithm not only significantly impacts the convergence rate but also affects the overall performance and generalization of the resulting models \cite{le2011optimization}. The landscape of optimization techniques has been traditionally dominated by variants of stochastic gradient methods, each adapted to different challenges posed by large-scale and complex networks \cite{duchi2011adaptive,kingma2014adam,polyak1964some,tieleman2012rmsprop,xie2022adan}.

Recently, the paradigm of learning to optimize (L2O) has emerged as a promising avenue, offering the potential to automate the discovery of optimizer designs through machine learning itself \cite{andrychowicz2016learning,li2016learning}. This approach leverages historical data and meta-learning techniques to design optimizers that are adapted to specific tasks or types of models, potentially outperforming hand-designed algorithms \cite{chen2022learning}.

Among the innovations in this field, the LION (evoLved sIgn mOmeNtum) optimizer represents a significant breakthrough \cite{chen2024symbolic}. Discovered by evolutionary search techniques within a predefined program space, LION utilizes a novel update mechanism that relies on the sign of the gradient. This design diverges from traditional adaptive learning rate strategies, positioning LION as a distinctly learned optimization algorithm.

Despite its innovative design and promising empirical performance, the convergence properties of LION have not been thoroughly analyzed, especially the \textit{convergence rate}, which poses a significant challenge for its adoption in safety-critical or regulated domains. The heuristic nature of its discovery process does not inherently guarantee convergence, making theoretical analysis both crucial and challenging \cite{chen2022learning}.

In nonconvex stochastic optimization, convergence rates are commonly characterized through upper bounds on gradient norms along the optimization trajectory. Most existing analyses focus on how the rate depends on the number of iterations $K$, with the goal of developing faster and more accurate optimization algorithms. However, as deep learning tasks continue to scale up, the problem dimension $d$ can also be extremely large. For example, in the GPT-3 training task, the number of model parameters can be as large as $d=1.75\times 10^{11}$ \cite{li2024frac}. In this regime, the convergence rates with $\bO(d)$ and $\bO(\sqrt{d})$ dependence may differ by several orders of magnitude, making the dependence on $d$ a crucial factor in convergence analysis for large scale models. 

This dimension-dependent perspective also makes it necessary to revisit how convergence rates are measured across different optimization algorithms. Adaptive learning rate methods, as one major branch of stochastic gradient methods, are typically analyzed under the gradient $\ell_2$ norm \cite{wang2023convergence,hong2024revisiting,zou2019sufficient,wang2024closing}, whereas sign-based methods are often analyzed under the gradient $\ell_1$ norm \cite{bernstein2018signsgd,safaryan2021stochastic,sun2023momentum}. In the early stages of deep learning research, this distinction had little effect on the dependence of convergence rates on the number of iterations $K$, since all norms are equivalent in finite-dimensional spaces. However, for a gradient vector $\nabla f(\btheta)\in\mathbb{R}^d$, the $\ell_1$ and $\ell_2$ norms satisfy
\begin{equation}\label{eqnorm}
    \|\nabla f(\btheta)\|_2\leq\|\nabla f(\btheta)\|_1\leq\sqrt{d}\|\nabla f(\btheta)\|_2.
\end{equation}
Therefore, once the dependence on the problem dimension $d$ is taken into account, the choice of gradient norm can directly affect the interpretation of a convergence rate comparison.

In this paper, we address this gap by providing a comprehensive convergence rate analysis of the LION optimizer in both constrained and unconstrained settings. Our analysis reveals critical insights into the convergence behavior of LION, identifying specific conditions under which it guarantees convergence. We demonstrate that the convergence rate of LION under the gradient $\ell_1$ norm, expressed as \(\bO(\sqrt{d}K^{-1/4})\), where \(K\) is the number of iterations, not only exhibits the currently optimal dependency on the problem dimension but also tightly matches the theoretical lower bounds established for nonconvex stochastic optimization algorithms on the number of iterations.
Given that the classical lower bound is measured by the gradient $\ell_2$ norm, we conduct extensive experiments to confirm that the gradient \(\ell_1/\ell_2\) norm ratio aligns with \(\Theta(\sqrt{d})\), thereby substantiating that our convergence rate aligns well with the theoretical lower bound concerning \(d\) in an empirical context.

Our findings provide an in-depth view of LION's convergence properties, enhancing its theoretical understanding and paving the way for more informed and reliable applications in practical deep learning scenarios.

\noindent\textbf{Paper organization.} The remainder of this paper is structured as follows: Section \ref{sec:bg} introduces the background, lists prior studies with a focus on the characteristics of the LION optimizer, and summarizes the contributions of this work. Section \ref{sec:ca} introduces our main contributions, detailing the theoretical convergence rate analysis of the LION optimizer. Section \ref{sec:exp} describes the experiments conducted to validate our findings. Section \ref{sec:c} summarizes the key points and concludes our study.
\section{Background}\label{sec:bg}
Suppose that we have an objective function $f$ that takes $\btheta$ as the input. Given a dataset $S$ comprising i.i.d. samples $\{\bzeta^1,\cdots,\bzeta^s\}$ of the random variable $\bzeta$ from the distribution $\mathcal{D}$, the minimization of $f$ corresponds to the learning process in most of the deep learning problems:
\begin{align}\label{unconstrainedproblem}
\min_{\btheta\in\R^d} f(\btheta)=\E_{\bzeta\sim\mathcal{D}} \left[\hat{f}(\btheta,\bzeta)\right],
\end{align}
where $d$ is the problem dimension and $\hat{f}$ is the loss function with respect to the model parameter $\btheta$ and the realization of the random variable $\bzeta$.
\subsection{Notations}
Throughout this paper, we use boldface to denote vectors and standard non-bold type to represent scalars. For vector $\x$, we use $\x^k$ to represent the value of $\x$ at the iteration $k$, while $\x_i$ stands for its ${i}^{th}$ component. We denote the model parameter as $\btheta$ and its gradient estimation as $\g$. For scalar $s$, $s^p$ still denotes its ${p}^{th}$ power. The $\ell_p$ norm is denoted as $\|\cdot\|_p$, and we use $\|\cdot\|$ to denote the Euclidean norm for brevity. For optimality, denote $f^\star=\inf_{\btheta\in\mathcal{F}}f(\btheta)$ where $\mathcal{F}$ is the feasible region, although the optimization algorithm may not converge to $f^\star$ in nonconvex settings. For asymptotics, we use $f(n)=\bO (g(n))$ to represent the existence of constants $c$ and $N$, such that $|f(n)| \leq c |g(n)|$ always holds for all \( n \geq N \); and use $f(n)=\Theta (g(n))$ to represent the existence of $c_1$, $c_2$, and $N$, such that $c_1 |g(n)| \leq |f(n)| \leq c_2 |g(n)|$ holds for all \( n \geq N \). We also use $\operatorname{polylog}(n)$ to denote a polylogarithmic factor in $n$.

\subsection{Related Work}
The convergence properties of stochastic optimization algorithms have been extensively investigated across numerous studies. Under the Assumptions \ref{assumption1} to \ref{assumption3} (see Section \ref{sec:ca}), the convergence rate of the vanilla Stochastic Gradient Descent (SGD) algorithm has been shown to be \cite{bottou2018optimization}:
\begin{equation}\label{lowerbound}
\frac{1}{K} \sum_{k=1}^{K} \E \left[ \left\| \nabla f(\btheta^k) \right\|_2 \right] \leq \bO \left( \frac{\sqrt[4]{\sigma^2 L (f(\btheta^1) - f^\star)}}{K^{1/4}} \right),
\end{equation}
which matches the theoretical lower bound of nonconvex stochastic optimization algorithms \cite{arjevani2023lower} with respect to $K$. 

Building upon the foundational results of SGD, several works aim to establish convergence rates for adaptive algorithms. 
For AdaGrad \cite{duchi2011adaptive}, Wang et al. \cite{wang2023convergence}, Liu et al. \cite{liu2023high}, and Hong et al. \cite{hong2024revisiting} provide convergence rate analyses. However, the first two results are established for AdaGrad-Norm variant \cite{ward2020adagrad}, which is not strictly equivalent to AdaGrad itself, and they do not give an explicit dependence on the problem dimension $d$. Hong et al. \cite{hong2024revisiting} analyze AdaGrad itself, but their rate under the gradient $\ell_2$ norm is worse than ours by a factor of $\bO(\sqrt{d})$ in terms of dimension dependence after simplification. Moreover, these AdaGrad results contain an additional $\operatorname{polylog}(K)$ factor, and therefore do not match the lower bound for stochastic gradient methods with respect to $K$.
For RMSProp \cite{tieleman2012rmsprop}, Zou et al. \cite{zou2019sufficient} and Shi et al. \cite{shi2021rmsprop} provide convergence rate results that suffer from similar limitations as the AdaGrad analyses discussed above. Li et al. \cite{li2024frac} explicitly characterize the dependence of RMSProp on both $K$ and $d$, but RMSProp is now less commonly used in current large-scale deep learning practice, which limits the practical guidance of this result. For Adam \cite{kingma2014adam}, Zhang et al. \cite{zhang2022adam}, Zou et al. \cite{zou2019sufficient}, and Wang et al. \cite{wang2024closing} provide convergence rate analyses. However, their rates still differ from the lower bound by an additional $\bO(\operatorname{polylog}(K))$ factor, and due to the complexity of Adam, these works do not rigorously specify the dependence on the problem dimension $d$.

Another line of work focuses on examining SignSGD methods, which incorporate a sign operation in the gradient update step. Unlike SGD and most adaptive learning rate methods, these algorithms are often analyzed under the gradient $\ell_1$ norm. Bernstein et al. firstly formalize the SignSGD algorithm and successfully adapt it for usage in distributed settings \cite{bernstein2018signsgd}. Although their work includes a convergence proof, the stipulation that SignSGD requires an increasing batch size diminishes the practical applicability in real-world applications. Later, researchers propose SignSGD-EF \cite{karimireddy2019error} and SignSGD-CF \cite{safaryan2021stochastic} algorithms to disentangle the relationship of its convergence property with the increasing batch size. Sun et al. point out that the momentum version of SignSGD will converge under weaker assumptions \cite{sun2023momentum}, where they bound the $\ell_1$ norm of the gradient under the same assumptions as ours:
\begin{equation}
\frac{1}{K} \sum_{k=1}^{K} \E \left[ \left\| \nabla f(\btheta^k) \right\|_1 \right] \leq \bO \left( \frac{f(\btheta^1) - f^\star}{LK^{1/4}} + \frac{d}{K^{1/4}}\right),
\end{equation}
which is close to \eqref{lowerbound}. 

The LION algorithm \cite{chen2024symbolic}, which also belongs to the SignSGD group, has had its convergence properties studied as well. Xiao et al. introduces a general framework for SGD-type methods with the LION included \cite{xiao2023convergence}. While they give a convergence proof that applies consistently across all methods, no results about the rates of convergence are provided. Chen et al. introduce a novel Lyapunov function to show the optimization dynamics of LION, demonstrating that LION can be seen as a principled method for minimizing a general loss function subject to box constraints \cite{chen2023lion}:
\begin{equation}\label{constrainedproblem}
\min f(\btheta),\quad s.t.\quad \|\btheta\|_{\infty}\leq\frac{1}{\lambda}.
\end{equation}
This formulation can be regarded as a type of regularization, which is a widely utilized technique in deep learning to enhance generalization and prevent overfitting \cite{kukavcka2017regularization,tian2022comprehensive}.

Liu et al. propose a pipeline for deploying LION under the distributed settings \cite{liu2024communication}. While this work proves the convergence property of LION for the constrained problem \eqref{constrainedproblem}, our analysis delineates a more precise convergence \textit{rate} that matches the theoretical lower bound on variable $K$. Also, we further provide the convergence rate result on the general unconstrained problem \eqref{unconstrainedproblem}.

\subsection{Contributions}
Our main contribution in this paper is a comprehensive convergence rate analysis of the LION optimizer across both constrained and unconstrained settings, demonstrating that LION achieves a convergence rate of $\bO(\sqrt{d}K^{-1/4})$ in both scenarios. For the constrained optimization problem, our findings show that LION's convergence rate meets the existing theoretical lower bounds for nonconvex stochastic optimization algorithms with respect to $K$, and offers currently optimal dependence on the problem dimension $d$. In the unconstrained setting, we provide the first convergence proof for LION, establishing a rate of convergence that shares the advantageous properties observed in the constrained case. In addition, our convergence rate is built on the gradient $\ell_1$ norm, whereas the lower bound in \eqref{lowerbound} is stated under the gradient $\ell_2$ norm. As indicated by \eqref{eqnorm}, this distinction matters for dimension dependence, so we carry out a number of experiments to empirically confirm that the gradient $\ell_1/\ell_2$ norm ratio aligns with $\Theta(\sqrt{d})$. Thus, our rate also matches \eqref{lowerbound} with respect to $d$ in the empirical sense. These results not only validate the effective performance of LION but also offer critical theoretical insights in practical optimization tasks.

\section{Convergence Rate Analysis}\label{sec:ca}
The update rule of the LION optimizer is formalized in Algorithm \ref{originallion} \cite{chen2024symbolic}.
\begin{algorithm}[H]
   \caption{LION}
   \label{originallion}
\begin{algorithmic}[1]
   \STATE Initialize $\btheta^1$, $\m^0$
   \FOR{$k=1,2,\cdots,K$}
   \STATE $\g^k\leftarrow\nabla \hat{f}(\btheta^k,\bzeta^k)$
   \STATE $\c^k\leftarrow\beta_1\m^{k-1}+(1-\beta_1)\g^k$
   \STATE $\btheta^{k+1}\leftarrow\btheta^k-\eta\left(\sign(\c^k)+\lambda\btheta^k\right)$
   \STATE $\m^k\leftarrow\beta_2\m^{k-1}+(1-\beta_2)\g^k$
   \ENDFOR
\end{algorithmic}
\end{algorithm}

Here $\eta$ is the learning rate, $\lambda$ is the decoupled weight decay coefficient, and $\beta_1,\beta_2$ are momentum interpolation hyperparameters.

To rigorously analyze the convergence properties of the LION optimizer, we first make several foundational assumptions that are standard across deep learning optimization literature, which are commonly employed in analyses of well-known optimizers such as SGD \cite{arjevani2023lower}, AdaGrad \cite{wang2023convergence,hong2024revisiting,liu2023high}, RMSProp \cite{zou2019sufficient,li2024frac}, and Adam \cite{reddi2019convergence,wang2024convergence}:

\noindent\assumption{assumption1} (L-smoothness)\textbf{.} The objective $f$ has a Lipschitz continuous gradient: for all $\x$ and $\y$, it holds that
\begin{align}
    \|\nabla f(\x)-\nabla f(\y)\|\leq L\|\x-\y\|.
\end{align}

\noindent\assumption{assumption2} (Unbiased gradient estimator)\textbf{.} The gradient estimate $\g^k$ is an unbiased estimator of the true objective gradient $\nabla f(\btheta)$:
\begin{align}
    \E_k\left[\g^k\right]=\nabla f(\btheta^k).
\end{align}

\noindent\assumption{assumption3} (Bounded noise variance)\textbf{.} The variance of the difference between the true gradient value and its estimation is bounded:
\begin{align}
    \E_k\left[\left\|\g^k-\nabla f(\btheta^k)\right\|^2\right]\leq \sigma^2.
\end{align}

\subsection{Constrained Case}
Consider the constrained problem \eqref{constrainedproblem}. First we introduce a lemma that is the instantiation of the work of Xie et al. \cite{xie2024implicit}:
\begin{lemma}[Lemma 3.8 in \cite{xie2024implicit} with $\ell_\infty$ norm] \label{lemma:kkt-point}
$\btheta^\star$ is a KKT point of problem (\ref{constrainedproblem}) iff
\begin{align}
\begin{aligned}\notag
\|\btheta^\star\|_{\infty}\leq \frac{1}{\lambda},\quad \lambda\<\btheta^\star,\nabla f(\btheta^\star)\> + \left\|\nabla f(\btheta^\star)\right\|_1=0.
\end{aligned}
\end{align}
\end{lemma}
It turns out that, given the properly initialized $\btheta^1$, the iteration sequence $\left\{\btheta^k\right\}$ of Algorithm \ref{originallion} would always stay in the feasible region, while keeping the term $\lambda\<\nabla f(\btheta^k),\btheta^k\>+\left\|\nabla f(\btheta^k)\right\|_1$ nonnegative. Together with Lemma \ref{lemma:kkt-point} and the upper bound provided in Theorem \ref{constrainedtheorem}, the LION algorithm converges to a KKT point of problem \eqref{constrainedproblem} with at most $\bO(\sqrt{d}K^{-1/4})$ iterations, as indicated in Corollary \ref{constraintcorollary}.

\begin{theorem}\label{constrainedtheorem}
Suppose that Assumptions \ref{assumption1}-\ref{assumption3} hold. Letting $\beta_1=1-\frac{c_1}{\sqrt{K}}$, $\beta_2=1-\frac{c_2}{\sqrt{K}}$, $\eta=\frac{c_3}{\sqrt{d}K^{3/4}}$, and $\|\btheta^1\|_{\infty}\leq \frac{1}{\lambda}$, where $c_1$, $c_2$, and $c_3$ are independent of $K$ and $d$. Then for Algorithm \ref{originallion}, it holds for every $k$ that
\begin{align}
\begin{aligned}\notag
(i)\quad \|\btheta^k\|_{\infty}\leq \frac{1}{\lambda},\quad \lambda\<\nabla f(\btheta^k),\btheta^k\>+\left\|\nabla f(\btheta^k)\right\|_1\geq 0,
\end{aligned}
\end{align}
and
\begin{align}
\begin{aligned}\notag
(ii)\quad&\frac{1}{K}\sum_{k=1}^K\E\left[\lambda\<\nabla f(\btheta^k),\btheta^k\>+\left\|\nabla f(\btheta^k)\right\|_1\right]\\
&\leq  \frac{\left(f(\btheta^1)-f^\star\right)\sqrt{d}}{c_3K^{1/4}}+\frac{2\sigma\sqrt{d}}{c_2K^{1/2}} + \frac{4Lc_3\sqrt{d}}{c_2K^{1/4}}+ \frac{2\sigma(2c_1+c_2)\sqrt{d}}{\sqrt{c_2}K^{1/4}}+\frac{2Lc_3\sqrt{d}}{K^{3/4}}.
\end{aligned}
\end{align}
\end{theorem}
Since the constants $c_1$, $c_2$, and $c_3$ are independent of $K$ and $d$, Theorem \ref{constrainedtheorem} translates into Corollary \ref{constraintcorollary} under some specifications.

\begin{corollary}\label{constraintcorollary}
Under the settings of Theorem \ref{constrainedtheorem}, if we set $c_1=c_2=\frac{\sqrt{L(f(\btheta^1)-f^\star)}}{\sigma}$, $c_3=\frac{(f(\btheta^1)-f^\star)^{3/4}}{L^{1/4}\sigma^{1/2}}$, and $K\geq\max\left\{\frac{\sigma^6}{L^3(f(\btheta^1)-f^\star)^3}, \frac{L(f(\btheta^1)-f^\star)}{\sigma^2}\right\}$, then we have
\begin{align}
\begin{aligned}\notag
\frac{1}{K}\sum_{k=1}^K\E\left[\lambda\<\nabla f(\btheta^k),\btheta^k\>+\left\|\nabla f(\btheta^k)\right\|_1\right]\leq \frac{15\sqrt{d}}{K^{1/4}}\sqrt[4]{\sigma^2 L(f(\btheta^1)-f^\star)}.
\end{aligned}
\end{align}
\end{corollary}

Corollary \ref{constraintcorollary} reflects that the LION algorithm converges at the same rate of SGD for the constrained minimization problem \eqref{constrainedproblem} with respect to the number of iterations. Comparing with \eqref{lowerbound}, one can easily observe that the bound in Corollary \ref{constraintcorollary} well matches the theoretical lower bound of nonconvex stochastic algorithms in terms of $K$. The only remaining aspect is $\sqrt{d}$; although a specific lower bound for this term has not been established for the constrained problem \eqref{constrainedproblem}, our analysis provides the currently optimal dependence on the problem dimension $d$. 
Additionally, our inequality bounds $\left\|\nabla f(\btheta^k)\right\|_1$, whereas \eqref{lowerbound} uses $\left\|\nabla f(\btheta^k)\right\|_2$. As shown in \eqref{eqnorm}, these two norms may differ by a factor of up to $\sqrt{d}$. Define the gradient norm ratio as
\begin{align}
    r=\frac{\left\|\nabla f(\btheta)\right\|_1}{\left\|\nabla f(\btheta)\right\|_2},
\end{align}
in Section \ref{sec:exp}, we perform a series of experiments on various vision and language tasks, empirically showing that the norm ratio $r$ stays close to the level of $\Theta(\sqrt{d})$, rather than the constant level $\Theta(1)$. Thus, our bound also matches \eqref{lowerbound} with respect to $d$ in the empirical sense. 

\subsection{Unconstrained Case}\label{sec:unccase}
Here we present another contribution of our work: the first rigorous guarantee of convergence for the LION optimizer under the unconstrained problem setting, along with a detailed analysis of the convergence rate. To derive the result, we exploit the decoupled weight decay term in LION and impose a finite-horizon condition on the weight decay parameter. Under this condition, the iterates admit a direct pathwise bound, as formalized in the Lemma \ref{lemma:finite-horizon-bound}.

\begin{lemma}\label{lemma:finite-horizon-bound}
Suppose that $\eta\lambda\leq \frac{1}{2K}$ and $\|\btheta^1\|_{\infty}\leq \frac{1}{2\lambda}$. Then for Algorithm \ref{originallion}, it holds for every $k=1,2,\ldots,K$ that
\begin{align}
\begin{aligned}\notag
\lambda\|\btheta^k\|_{\infty}\leq \frac{3}{4}.
\end{aligned}
\end{align}
\end{lemma}

Lemma \ref{lemma:finite-horizon-bound} shows that the iteration points of LION never touch the box boundary of problem \eqref{constrainedproblem}, which essentially turns into unconstrained optimization. Our main findings are that LION converges to a stationary point of the unconstrained problem \eqref{unconstrainedproblem} with at most  $\bO(\sqrt{d}K^{-1/4})$ steps, provided by Corollary \ref{unconstraintcorollary}.

\begin{theorem}\label{theorem2}
Suppose that Assumptions \ref{assumption1}-\ref{assumption3} hold. Letting $\beta_1=1-\frac{c_1}{\sqrt{K}}$, $\beta_2=1-\frac{c_2}{\sqrt{K}}$, $\eta=\frac{c_3}{\sqrt{d}K^{3/4}}$, and $\eta\lambda\leq \frac{1}{2K}$, where $c_1$, $c_2$, and $c_3$ are independent of $K$ and $d$. Initialize $\|\btheta^1\|_{\infty}\leq \frac{1}{2\lambda}$, then for Algorithm \ref{originallion}, we have
\begin{align}
\begin{aligned}\notag
\frac{1}{4K}\sum_{k=1}^K\E\left[\left\|\nabla f(\btheta^k)\right\|_1\right]&\leq  \sqrt{d}\frac{f(\btheta^1)-f^\star}{c_3K^{1/4}}+\frac{2\sigma\sqrt{d}}{c_2K^{1/2}} + \frac{4Lc_3\sqrt{d}}{c_2K^{1/4}} \\&\quad+ \frac{2\sigma(2c_1+c_2)\sqrt{d}}{\sqrt{c_2}K^{1/4}}+\frac{2Lc_3\sqrt{d}}{K^{3/4}}.
\end{aligned}
\end{align}
\end{theorem}

Similarly, we have Corollary \ref{unconstraintcorollary} depicting LION's convergence rate to stationary point:
\begin{corollary}\label{unconstraintcorollary}
Under the settings of Theorem \ref{theorem2}, if we set $c_1=c_2=\frac{\sqrt{L(f(\btheta^1)-f^\star)}}{\sigma}$, $c_3=\frac{(f(\btheta^1)-f^\star)^{3/4}}{L^{1/4}\sigma^{1/2}}$, and $K\geq\max\left\{\frac{\sigma^6}{L^3(f(\btheta^1)-f^\star)^3},\frac{L(f(\btheta^1)-f^\star)}{\sigma^2}\right\}$, then we have
\begin{align}
\begin{aligned}\notag
\frac{1}{4K}\sum_{k=1}^K\E\left[\left\|\nabla f(\btheta^k)\right\|_1\right]\leq  \frac{15\sqrt{d}}{K^{1/4}}\sqrt[4]{\sigma^2 L(f(\btheta^1)-f^\star)}.
\end{aligned}
\end{align}
\end{corollary}

Corollary \ref{unconstraintcorollary} implies that the LION algorithm converges to a stationary point of problem \eqref{unconstrainedproblem} at the same rate of SGD. Again, our inequality bounds the $\ell_1$ norms of the gradients, matching the lower bound \eqref{lowerbound} in the empirical sense.

\section{Experiments}\label{sec:exp}
In this section, we conduct extensive experiments on vision and language tasks to provide empirical evidence of the superiority of LION over SGD (which is known to match the theoretical lower bound) in terms of both training loss and evaluation performance. Also, we track the gradient norm ratio $r$ for each of the tasks, confirming that this ratio increases proportionally to the square root of the model size $d$, thus prove our convergence rate matches the lower bound \eqref{lowerbound} with respect to $d$ in the empirical sense.

For vision tasks, we train ResNet18 \cite{he2016deep}, ResNet50 \cite{he2016deep}, and ViT-S \cite{dosovitskiy2020image} on the datasets of CIFAR-10 \cite{krizhevsky2009learning}, CIFAR-100 \cite{krizhevsky2009learning}, and ImageNet-1K \cite{ILSVRC15}, resulting in 9 vision experiments. Let $|S|$ denote the dataset size. Unlike the stochastic training paradigm, we compute the average sample loss on the \textit{full} training dataset:
\begin{align}\label{fullview}
    f(\btheta)=\frac{1}{|S|}\sum_{k=1}^{|S|}\hat{f}(\btheta,\bzeta_k),
\end{align}
and compute the \textit{full} gradient $\nabla f(\btheta)$, to obtain a complete view on dataset $S$ and a noiseless measurement over the gradient norm ratio $r$. This is achieved by interleaving training and logging epochs. The training epoch performs a typical stochastic training loop with the model parameters updated. During the logging epoch, since the gradient of \eqref{fullview} is linear over samples, we accumulate loss and gradient information across batches over the entire dataset while keeping the parameters frozen. After each logging epoch, we evaluate model performance by measuring the Top-1 image classification accuracy on the official test split.

\begin{figure}[tbhp]
    \centering
    \includegraphics[width=\textwidth]{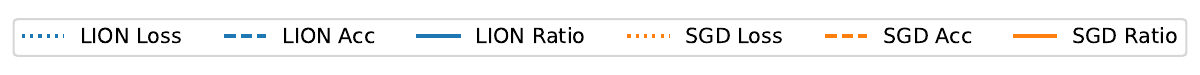}

    \begin{subfigure}{0.5\textwidth}
        \centering
        \includegraphics[width=\linewidth]{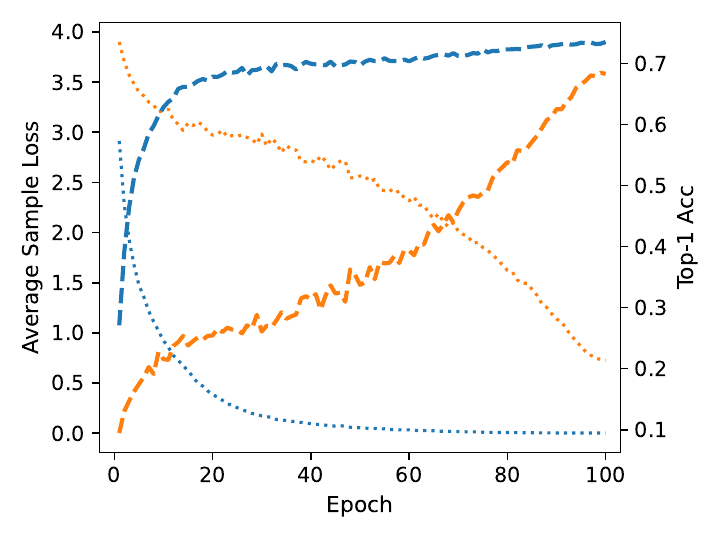}
        \caption{ResNet18 training loss and Top-1 accuracy}
        \label{fig:subfig1a-i}
    \end{subfigure}%
    \begin{subfigure}{0.5\textwidth}
        \centering
        \includegraphics[width=\linewidth]{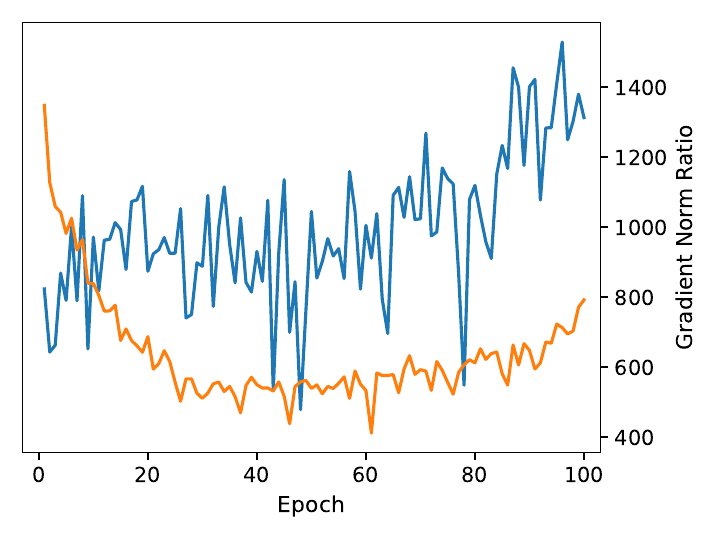}
        \caption{ResNet18 gradient norm ratio ($\sqrt{d}\approx3350$)}
        \label{fig:subfig1b-i}
    \end{subfigure}

    \begin{subfigure}{0.5\textwidth}
        \centering
        \includegraphics[width=\linewidth]{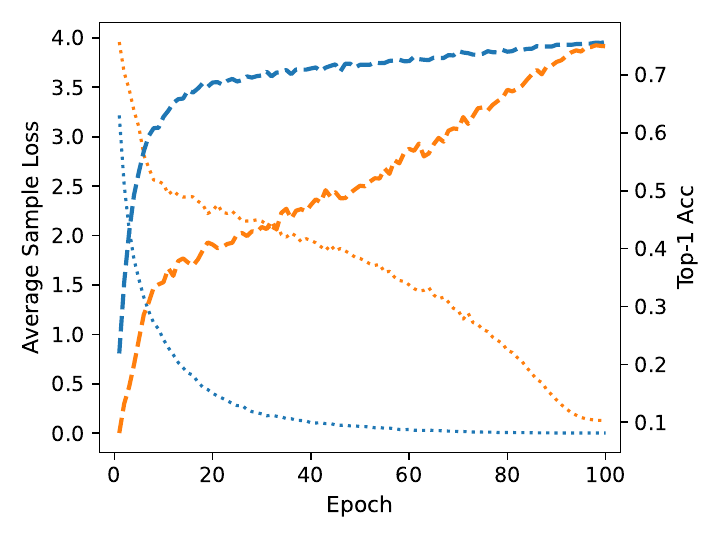}
        \caption{ResNet50 training loss and Top-1 accuracy}
        \label{fig:subfig2a-ii}
    \end{subfigure}%
    \begin{subfigure}{0.5\textwidth}
        \centering
        \includegraphics[width=\linewidth]{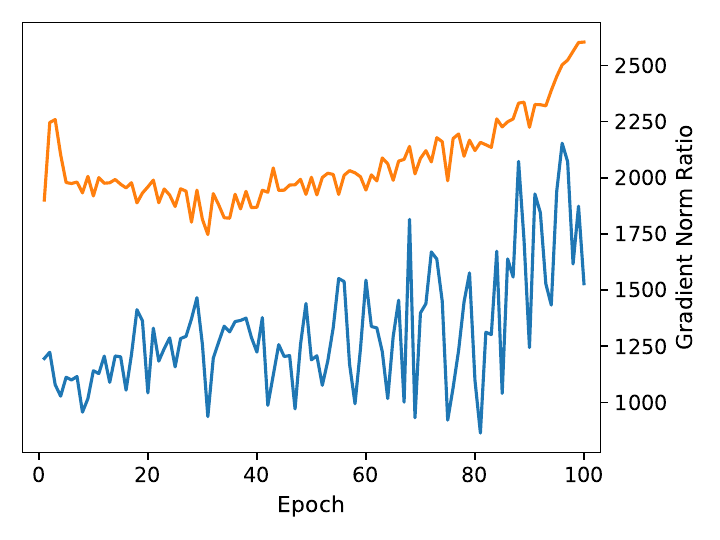}
        \caption{ResNet50 gradient norm ratio ($\sqrt{d}\approx 4869$)}
        \label{fig:subfig2b-ii}
    \end{subfigure}

    \begin{subfigure}{0.5\textwidth}
        \centering
        \includegraphics[width=\linewidth]{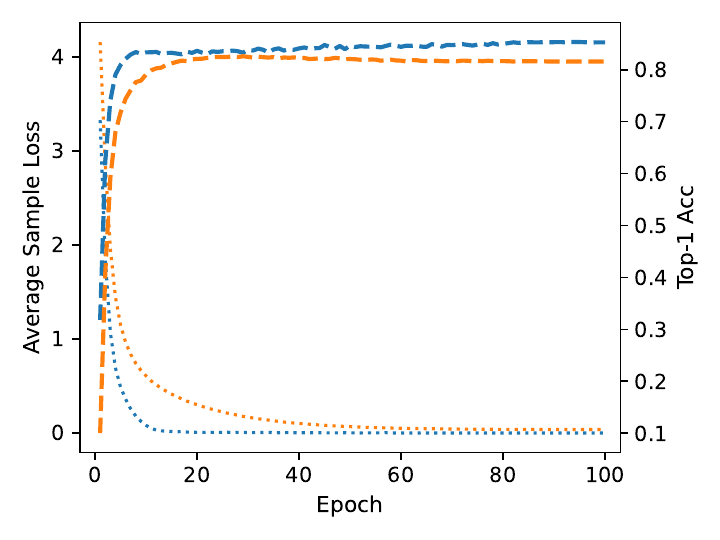}
        \caption{ViT-S training loss and Top-1 accuracy}
        \label{fig:subfig3a-iii}
    \end{subfigure}%
    \begin{subfigure}{0.5\textwidth}
        \centering
        \includegraphics[width=\linewidth]{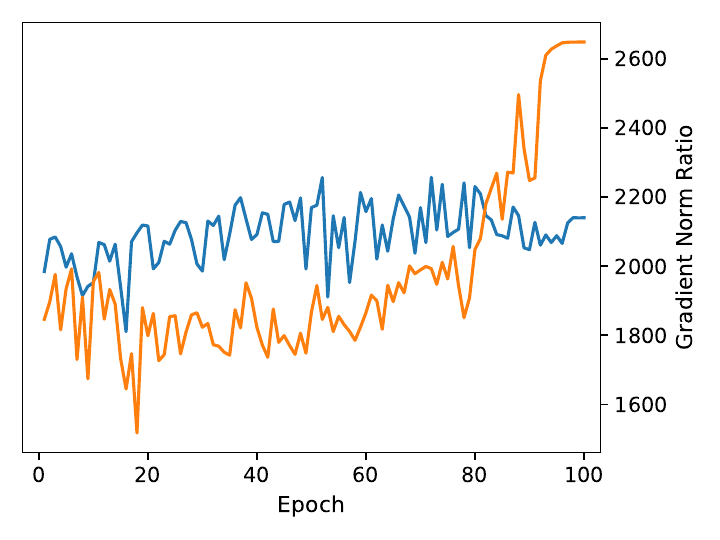}
        \caption{ViT-S gradient norm ratio ($\sqrt{d}\approx 4659$)}
        \label{fig:subfig3b-iii}
    \end{subfigure}
    \caption{Overview of results of ResNet18 \cite{he2016deep}, ResNet50 \cite{he2016deep}, and ViT-S \cite{dosovitskiy2020image} models training and evaluating on CIFAR-100 \cite{krizhevsky2009learning} dataset.
    Panels (a), (c), and (e) depict the training loss and Top-1 accuracy, and panels (b), (d), and (f) illustrate the gradient norm ratio.}
    \label{fig:cifar100}
\end{figure}

\begin{figure}[t!]
    \centering
    \includegraphics[width=\textwidth]{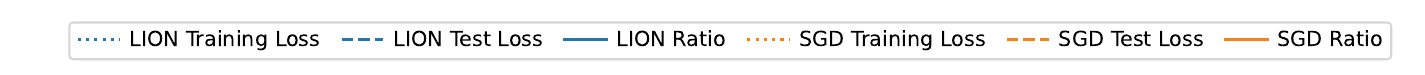}

    \begin{subfigure}{0.5\textwidth}
        \centering
        \includegraphics[width=\linewidth]{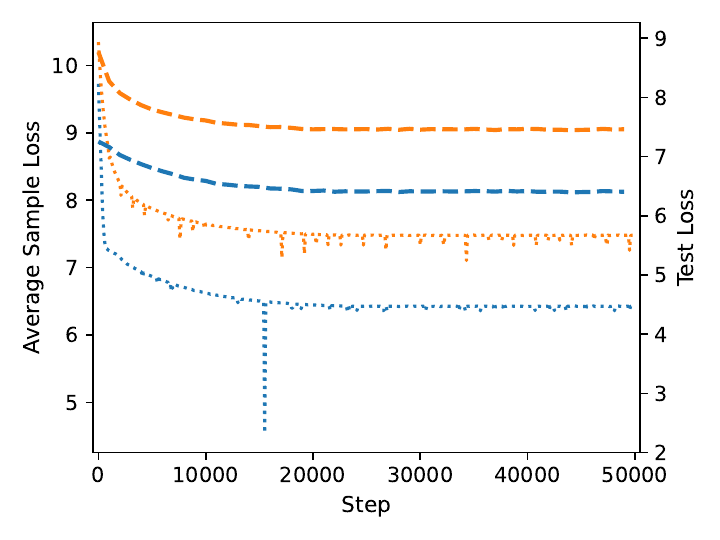}
        \caption{BERT-Small training and test loss}
        \label{fig:bert_small_loss}
    \end{subfigure}%
    \begin{subfigure}{0.5\textwidth}
        \centering
        \includegraphics[width=\linewidth]{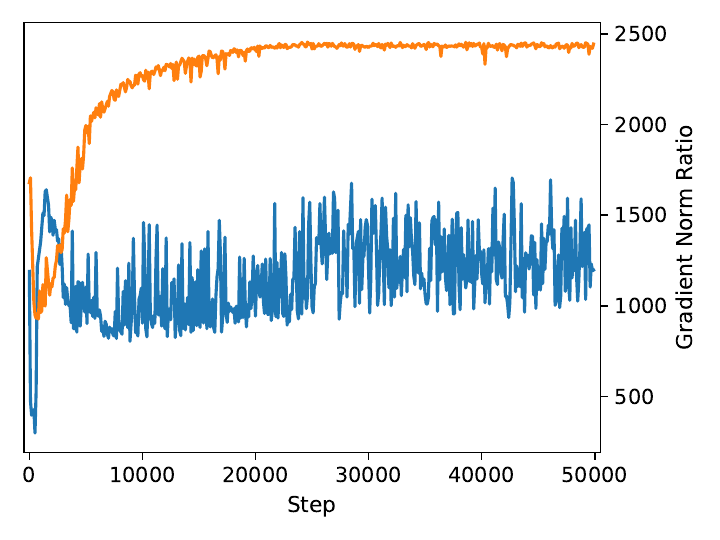}
        \caption{BERT-Small gradient norm ratio ($\sqrt{d}\approx 5394$)}
        \label{fig:bert_small_ratio}
    \end{subfigure}

    \begin{subfigure}{0.5\textwidth}
        \centering
        \includegraphics[width=\linewidth]{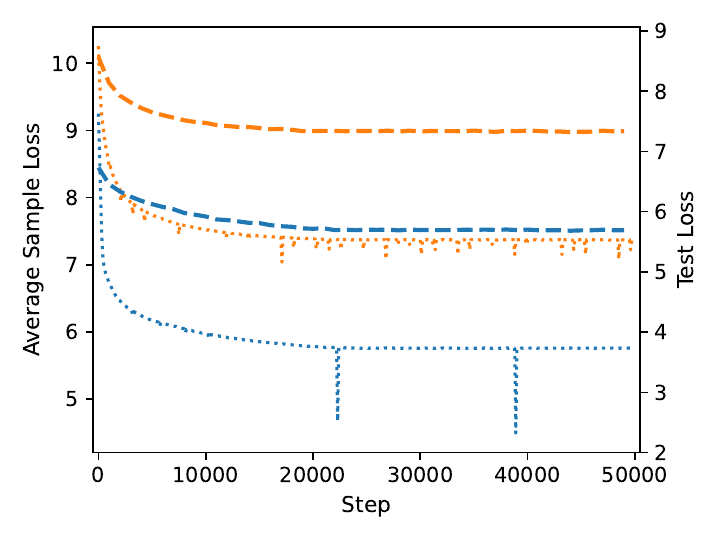}
        \caption{BERT-Base training and test loss}
        \label{fig:bert_base_loss}
    \end{subfigure}%
    \begin{subfigure}{0.5\textwidth}
        \centering
        \includegraphics[width=\linewidth]{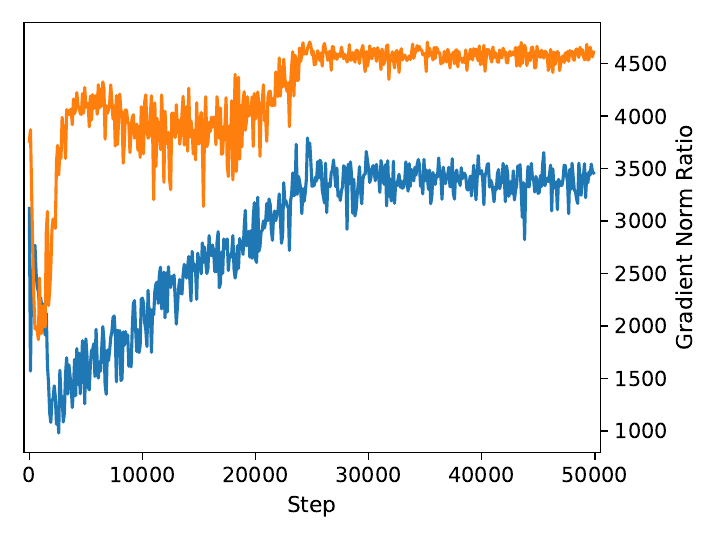}
        \caption{BERT-Base gradient norm ratio ($\sqrt{d}\approx 10496$)}
        \label{fig:bert_base_ratio}
    \end{subfigure}

    \caption{Overview of results for BERT-Small and BERT-Base models training and evaluating on the OpenWebText \cite{Gokaslan2019OpenWeb} dataset. Panels (a) and (c) depict the training loss and test loss, while panels (b) and (d) illustrate the gradient norm ratio.}
    \label{fig:bert}
\end{figure}

\begin{figure}[t!]
    \centering
    \includegraphics[width=\textwidth]{legend_language.pdf}

    \begin{subfigure}{0.5\textwidth}
        \centering
        \includegraphics[width=\linewidth]{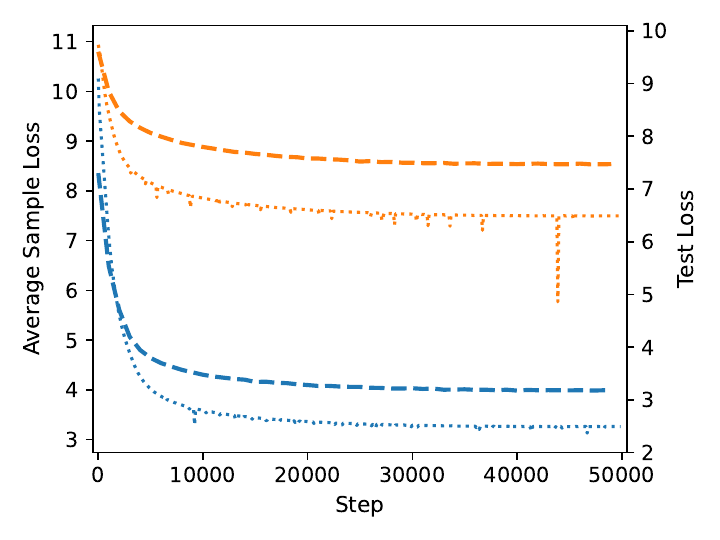}
        \caption{GPT-2 Small training and test loss}

    \end{subfigure}%
    \begin{subfigure}{0.5\textwidth}
        \centering
        \includegraphics[width=\linewidth]{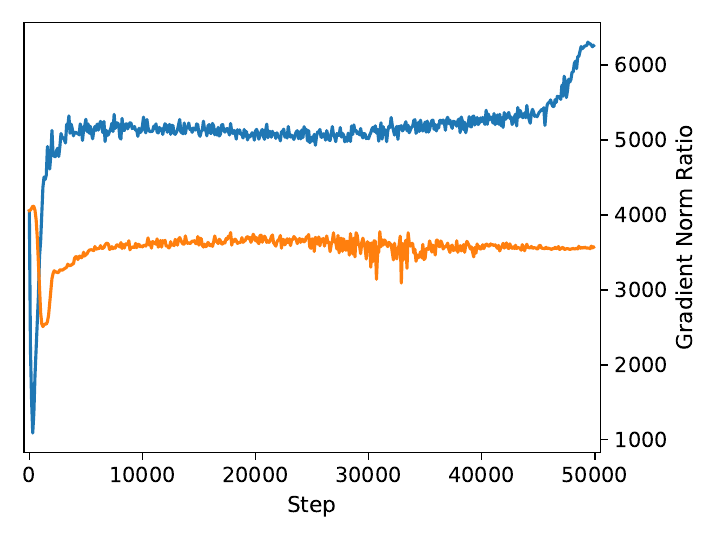}
        \caption{GPT-2 Small gradient norm ratio ($\sqrt{d}\approx 11148$)}

    \end{subfigure}

    \begin{subfigure}{0.5\textwidth}
        \centering
        \includegraphics[width=\linewidth]{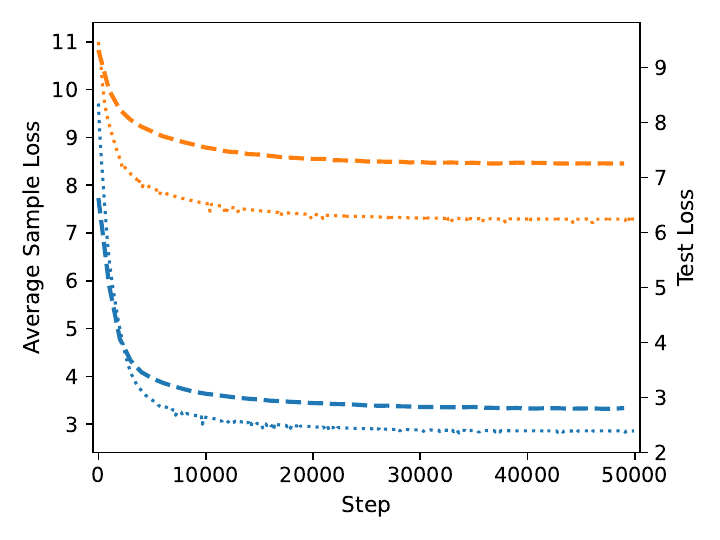}
        \caption{GPT-2 Medium training and test loss}

    \end{subfigure}%
    \begin{subfigure}{0.5\textwidth}
        \centering
        \includegraphics[width=\linewidth]{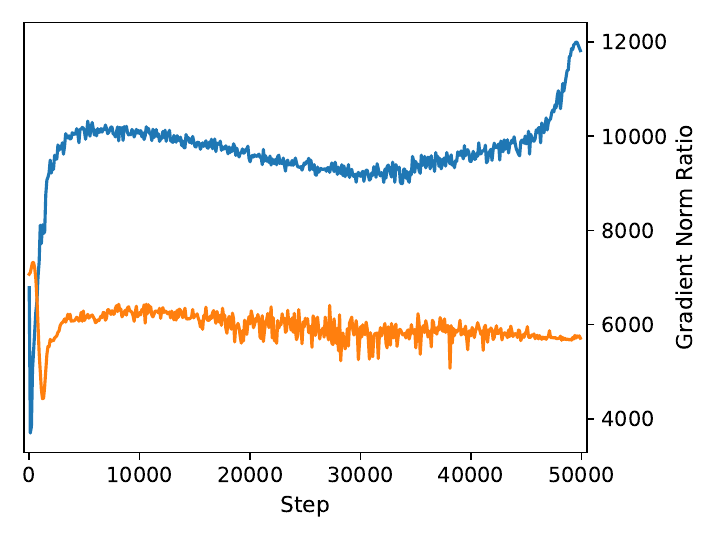}
        \caption{\raggedright\makebox[0pt][l]{GPT-2 Medium gradient norm ratio (\(\sqrt{d} \approx 18859\))}}
    \end{subfigure}

    \caption{Overview of results of GPT-2 \cite{radford2019language} Small and Medium models training and evaluating on the OpenWebText \cite{Gokaslan2019OpenWeb} dataset. Panels (a) and (c) depict the training loss and test loss, while panels (b) and (d) illustrate the gradient norm ratio.}
    \label{fig:gpt2}
\end{figure}
Figure \ref{fig:cifar100} presents the results of the ResNet18 \cite{he2016deep}, ResNet50 \cite{he2016deep}, and ViT-S \cite{dosovitskiy2020image} models on the CIFAR-100 \cite{krizhevsky2009learning} dataset, demonstrating that the LION optimizer consistently outperforms SGD by achieving lower losses and higher accuracies across all evaluated tasks. Moreover, an examination of panels (b), (d), and (f) indicates that the gradient norm ratio for LION closely correlates with the square root of the problem dimension, i.e., $r=\Theta(\sqrt{d})$. This behavior underscores LION's effective scalability and superior performance in complex optimization scenarios.

In addition, we extend our experimentation to include the CIFAR-10 \cite{krizhevsky2009learning} and ImageNet-1K \cite{ILSVRC15} datasets. The results are detailed in \ref{subsec:vt}. This comprehensive evaluation further corroborates our initial findings, demonstrating LION's robust performance across diverse tasks. For the sake of reproducibility and to provide a clear understanding of the experimental context, we provide the detailed training settings in \ref{sec:dts}.

For language tasks, we train the classic BERT-Small \cite{devlin2018bert}, BERT-Base \cite{devlin2018bert}, as well as the GPT-2 Small \cite{radford2019language} and the GPT-2 Medium \cite{radford2019language} models from scratch on OpenWebText \cite{Gokaslan2019OpenWeb} dataset. For the BERT models, we define the loss to be the sum of Masked Language Modeling (MLM) loss and Sentence Order Prediction (SOP) loss, and establish the performance metric to be the total loss on the test dataset. For the GPT-2 models, the loss is defined as the Language Model (LM) loss, which is the cross entropy loss between the predicted next-token probabilities and the actual tokens in the sequence. Similarly, we define the performance metric to be the LM loss on the test dataset.

Analogous to the vision tasks, we adapt our methodology to compute the loss and gradient for language tasks on the OpenWebText \cite{Gokaslan2019OpenWeb} dataset, which comprises about 9 billion tokens. Due to the computational constraints and the sheer volume of data, conducting a full pass to compute the exact loss and gradient is not feasible. Instead, we approximate the full loss $f(\btheta)$ and the full gradient $\nabla f(\btheta)$ by aggregating results across 100 batches. Further experiments show that extending the accumulation to 1000 or 5000 batches does not significantly alter the outcomes.

Results depicted in Figures \ref{fig:bert} and \ref{fig:gpt2} confirm that the LION optimizer yields superior performance in terms of both training and test losses when compared to SGD for BERT and GPT-2 models. The observed performance gap is substantial, as previous research shows that SGD is known to perform worse than Adam \cite{kingma2014adam} on transformers \cite{zhang2024transformers,kunstner2023noise,pan2023toward}, while LION is often shown to be on par with or exceed the performance of the Adam optimizer \cite{chen2024symbolic}. Also, the gradient norm ratio of LION again scales proportionally to $\sqrt{d}$, empirically confirming that $r=\Theta(\sqrt{d})$ holds for typical language tasks.

Additionally we assess the zero-shot performances of GPT-2 models on the WikiText-103 dataset \cite{merity2016pointer} at various intermediate training stages. The results indicate that LION still outperforms SGD by achieving lower perplexity in fewer training steps. We refer the reader to \ref{subsec:lt} for details.

In summary, the experimental results affirm the LION optimizer's distinct advantage over SGD, consistently delivering lower losses and elevated performance across various tasks. Additionally, the gradient norm ratio \( r \) scaling in alignment with \( \Theta(\sqrt{d}) \) across both vision and language models provides effective empirical support for its theoretical efficiency. These findings highlight the LION optimizer's capability in managing diverse computational demands.
\section{Conclusion}\label{sec:c}
This paper provides a comprehensive convergence rate analysis on the LION optimizer, a novel algorithm tailored for deep learning. We conclude that LION converges at the rate of $\bO(\sqrt{d}K^{-1/4})$ under the gradient $\ell_1$ norm, exactly matching the theoretical lower bound for stochastic optimization algorithms with respect to $K$ and achieving the currently best known dependence on the problem dimension $d$. Extensive experiments on vision and language tasks show that the gradient \(\ell_1/\ell_2\) norm ratio aligns with \(\Theta(\sqrt{d})\), confirming that our convergence rate also matches the theoretical lower bound with respect to \(d\) in the empirical sense. Our analysis covers both constrained and unconstrained settings, establishing robust theoretical support for the empirical effectiveness observed with LION. 
\section*{Acknowledgments}
Z. Lin was supported by National Key R\&D Program of China \linebreak (2022ZD0160300), the NSF China (No. 62276004), and Qualcomm. H. Li was supported by NSF China (No. 62476142 and 62006116).

\appendix
\section{Proof of the Theorems}\label{sec:p}
We complete the proofs by proving the Theorem \ref{constrainedtheorem} (i), bounding the $\bdelta^k=\c^k-\nabla f(\btheta^k)$, and finally proving the remaining theorems. 
\subsection{Proof of Theorem \ref{constrainedtheorem} (i)}
\begin{proof}
Borrowed from \cite{xie2024implicit}, from the update of $\btheta^k$, we have
\begin{align}
\begin{aligned}\notag
\|\btheta^{k+1}\|_{\infty}-\frac{1}{\lambda}=& \|(1-\eta\lambda)\btheta^k-\eta\sign(\c^k)\|_{\infty}-\frac{1}{\lambda}\\
\leq& (1-\eta\lambda)\|\btheta^k\|_{\infty}+\eta\|\sign(\c^k)\|_{\infty}-\frac{1}{\lambda}\\
=& (1-\eta\lambda)\|\btheta^k\|_{\infty}+\eta-\frac{1}{\lambda}\\
=& (1-\eta\lambda)\left(\|\btheta^k\|_{\infty}-\frac{1}{\lambda}\right),
\end{aligned}
\end{align}
From $\|\btheta^1\|_{\infty}\leq \frac{1}{\lambda}$, we have the first conclusion. For the second one, we have
\begin{align}
\begin{aligned}\notag
\lambda\<\nabla f(\btheta^k),\btheta^k\>&+\left\|\nabla f(\btheta^k)\right\|_1\geq -\lambda\left\|\nabla f(\btheta^k)\right\|_1\|\btheta^k\|_{\infty}+\left\|\nabla f(\btheta^k)\right\|_1\geq 0,
\end{aligned}
\end{align}
which completes the proof.
\end{proof}
\subsection{The Bounding Lemma}
\begin{lemma}\label{boundinglemma}
Suppose that Assumptions 1-3 hold. Denote $\bdelta^k=\c^k-\nabla f(\btheta^k)$. Then for Algorithm \ref{originallion}, we have
\begin{align}
\begin{aligned}\notag
\frac{1}{K}\sum_{k=1}^K\E\left[\|\bdelta^k\|\right]\leq& \frac{\sigma}{K(1-\beta_2)} + \frac{2L\eta\sqrt{d}}{1-\beta_2} + (|\beta_1-\beta_2|+1-\beta_1) \cdot\frac{\sigma}{\sqrt{1-\beta_2}}.
\end{aligned}
\end{align}
\end{lemma}

\begin{proof}
Further denote $\bxi^k=\g^k-\nabla f(\btheta^k)$. Using the update rule in Algorithm \ref{originallion}, we have
\begin{align}
\begin{aligned}\notag
\bdelta^k=&\beta_1\m^{k-1}+(1-\beta_1)\g^k-\nabla f(\btheta^k)\\
=&\beta_1\beta_2\m^{k-2}+\beta_1(1-\beta_2)\g^{k-1}+(1-\beta_1)\g^k-\nabla f(\btheta^k)\\
=&\beta_2(\c^{k-1}-(1-\beta_1)\g^{k-1})+\beta_1(1-\beta_2)\g^{k-1}+(1-\beta_1)\g^k-\nabla f(\btheta^k)\\
=&\beta_2(\bdelta^{k-1}+\nabla f(\btheta^{k-1}))-\beta_2(1-\beta_1)g^{k-1}+\beta_1(1-\beta_2)\g^{k-1}\\&\hspace*{6.9cm}+(1-\beta_1)\g^k-\nabla f(\btheta^k)\\
=&\beta_2(\bdelta^{k-1}+\nabla f(\btheta^{k-1}))+(\beta_1-\beta_2)(\bxi^{k-1}+\nabla f(\btheta^{k-1}))\\&\hspace*{6.9cm}+(1-\beta_1)(\bxi^{k}+\nabla f(\btheta^{k}))-\nabla f(\btheta^k)\\
=&\beta_2\bdelta^{k-1}-\beta_1(\nabla f(\btheta^k)-\nabla f(\btheta^{k-1}))+(\beta_1-\beta_2)\bxi^{k-1}+(1-\beta_1)\bxi^k\\
=&\beta_2^{k-1}\bdelta^1 + \sum_{t=2}^k\beta_2^{k-t}\Bigg(-\beta_1\left(\nabla f(\btheta^t)-\nabla f(\btheta^{t-1})\right)+(\beta_1-\beta_2)\bxi^{t-1}+(1-\beta_1)\bxi^t\Bigg)\\
=&\beta_2^{k-1}\bdelta^1 - \beta_1\sum_{t=2}^k\beta_2^{k-t}\left(\nabla f(\btheta^t)-\nabla f(\btheta^{t-1})\right)+(\beta_1-\beta_2)\sum_{t=2}^k\beta_2^{k-t}\bxi^{t-1} \\&\hspace*{7.6cm}+ (1-\beta_1)\sum_{t=2}^k\beta_2^{k-t}\bxi^t.
\end{aligned}
\end{align}
Taking expectations to get
\begin{align}
\begin{aligned}\notag
\E\left[\|\bdelta^k\|\right]\leq& \beta_2^{k-1}\E\left[\|\bdelta^1\|\right] + \beta_1\underbrace{\sum_{t=2}^k\beta_2^{k-t}\E\left[\left\|\nabla f(\btheta^t)-\nabla f(\btheta^{t-1})\right\|\right]}_{\text{\rm term (a)}}\\
&+\underbrace{\E\left[\left\| (\beta_1-\beta_2)\sum_{t=2}^k\beta_2^{k-t}\bxi^{t-1}+ (1-\beta_1)\sum_{t=2}^k\beta_2^{k-t}\bxi^t\right\|\right]}_{\text{\rm term (b)}}.
\end{aligned}
\end{align}
For term (a), using the first inequality of Theorem \ref{constrainedtheorem} (i), we have 
\begin{align}
\begin{aligned}\notag
\mbox{term (a)}\leq& L\sum_{t=2}^k\beta_2^{k-t}\E\left[\|\btheta^t-\btheta^{t-1}\|\right]\\
=& L\eta\sum_{t=2}^k\beta_2^{k-t}\E\left[\|\sign(\c^{t-1})+\lambda\btheta^{t-1}\|\right]\\
\leq& 2L\eta\sqrt{d}\sum_{t=2}^k\beta_2^{k-t}\\
\leq& \frac{2L\eta\sqrt{d}}{1-\beta_2}.
\end{aligned}
\end{align}
For term (b), we have
\begin{equation*}
\begin{aligned}
\text{term (b)} &\leq |\beta_1-\beta_2|\E\left[\left\|\sum_{t=2}^k\beta_2^{k-t}\bxi^{t-1}\right\|\right]+(1-\beta_1)\E\left[\left\|\sum_{t=2}^k\beta_2^{k-t}\bxi^{t}\right\|\right] \\
&\leq |\beta_1-\beta_2|\sqrt{\E\left[\left\|\sum_{t=2}^k\beta_2^{k-t}\bxi^{t-1}\right\|^2\right]}+(1-\beta_1)\sqrt{\E\left[\left\|\sum_{t=2}^k\beta_2^{k-t}\bxi^{t}\right\|^2\right]}\\
&=|\beta_1-\beta_2|\sqrt{\sum_{t=2}^k\beta_2^{2(k-t)}\E\left[\left\|\bxi^{t-1}\right\|^2\right]}+(1-\beta_1)\sqrt{\sum_{t=2}^k\beta_2^{2(k-t)}\E\left[\left\|\bxi^{t}\right\|^2\right]}\\
&=|\beta_1-\beta_2|\sqrt{\sigma^2\sum_{t=2}^k\beta_2^{2(k-t)}}+(1-\beta_1)\sqrt{\sigma^2\sum_{t=2}^k\beta_2^{2(k-t)}}\\
&\leq\left(|\beta_1-\beta_2|+(1-\beta_1)\right)\cdot\frac{\sigma}{\sqrt{1-\beta_2^2}}\\
&\leq\left(|\beta_1-\beta_2|+(1-\beta_1)\right)\cdot\frac{\sigma}{\sqrt{1-\beta_2}}.
\end{aligned}
\end{equation*}
Plugging terms (a) and (b) back, we get
\begin{align}
\begin{aligned}\notag
\E\left[\|\bdelta^k\|\right]\leq&\beta_2^{k-1}\E\left[\|\bdelta^1\|\right] + \frac{2L\eta\sqrt{d}}{1-\beta_2} + \left(|\beta_1-\beta_2|+(1-\beta_1)\right)\cdot\frac{\sigma}{\sqrt{1-\beta_2}},
\end{aligned}
\end{align}
and
\begin{align}
\begin{aligned}\notag
\frac{1}{K}\sum_{k=1}^K\E\left[\|\bdelta^k\|\right]\leq& \frac{1}{K(1-\beta_2)}\E\left[\|\bdelta^1\|\right] +\frac{2L\eta\sqrt{d}}{1-\beta_2}+ \left(|\beta_1-\beta_2|+(1-\beta_1)\right)\cdot\frac{\sigma}{\sqrt{1-\beta_2}}.
\end{aligned}
\end{align}
Initializing $\m^0=\g^1$, we have $\E\left[\|\bdelta^1\|\right]=\E\left[\|\g^1-\nabla f(\btheta^1)\|\right]\leq\sigma$, which completes the proof.
\end{proof}
\subsection{Proof of Theorem \ref{constrainedtheorem} (ii)}\label{pf}
\begin{proof}
As the gradient is Lipschitz, we have
\begin{align}
\begin{aligned}\notag
f&(\btheta^{k+1})-f(\btheta^k)\\\leq&\<\nabla f(\btheta^k),\btheta^{k+1}-\btheta^k\>+\frac{L}{2}\|\btheta^{k+1}-\btheta^k\|^2\\
=&-\eta\<\nabla f(\btheta^k),\sign(\c^k)+\lambda\btheta^k\>+\frac{L\eta^2}{2}\|\sign(\c^k)+\lambda\btheta^k\|^2\\
=&-\eta\lambda\<\nabla f(\btheta^k),\btheta^k\>-\eta\<\nabla f(\btheta^k),\sign(\nabla f(\btheta^k))\>\\&-\eta\<\nabla f(\btheta^k),\sign(\c^k)-\sign(\nabla f(\btheta^k))\>+\frac{L\eta^2}{2}\|\sign(\c^k)+\lambda\btheta^k\|^2\\
=&-\eta\lambda\<\nabla f(\btheta^k),\btheta^k\>-\eta\left\|\nabla f(\btheta^k)\right\|_1\\&-\eta\<\nabla f(\btheta^k),\sign(\c^k)-\sign(\nabla f(\btheta^k))\>+\frac{L\eta^2}{2}\|\sign(\c^k)+\lambda\btheta^k\|^2\\
\leq&-\eta\lambda\<\nabla f(\btheta^k),\btheta^k\>-\eta\left\|\nabla f(\btheta^k)\right\|_1\\&+\eta\sum_{i=1}^d\left|\nabla_i f(\btheta^k)\right|\left|\sign(\c_i^k)-\sign(\nabla_i f(\btheta^k))\right|+2dL\eta^2.
\end{aligned}
\end{align}

\noindent If $\sign(\c_i^k)=\sign(\nabla_i f(\btheta^k))$, we have
\begin{align}
\begin{aligned}\notag
\left|\nabla_i f(\btheta^k)\right|\left|\sign(\c_i^k)-\sign\left(\nabla_i f(\btheta^k)\right)\right|=0.
\end{aligned}
\end{align}

\noindent If $\sign(\c_i^k)\neq\sign(\nabla_i f(\btheta^k))$, we have $\c_i^k\nabla_i f(\btheta^k)\leq 0$ and
\begin{align}
\begin{aligned}\notag
\left|\nabla_i f(\btheta^k)\right|&\left|\sign(\c_i^k)-\sign\left(\nabla_i f(\btheta^k)\right)\right|= 2|\nabla_i f(\btheta^k)|\leq 2|\nabla_i f(\btheta^k)-\c_i^k|.
\end{aligned}
\end{align}
Combing the above two cases and denoting $\bdelta^k=\c^k-\nabla f(\btheta^k)$, we have
\begin{align}
\begin{aligned}\notag
f&(\btheta^{k+1})-f(\btheta^k)\\\leq&-\eta\lambda\<\nabla f(\btheta^k),\btheta^k\>-\eta\left\|\nabla f(\btheta^k)\right\|_1+2\eta\|\bdelta^k\|_1+2dL\eta^2\\
\leq&-\eta\lambda\<\nabla f(\btheta^k),\btheta^k\>-\eta\left\|\nabla f(\btheta^k)\right\|_1+2\eta\sqrt{d}\|\bdelta^k\|+2dL\eta^2.
\end{aligned}
\end{align}
Taking expectations, summing over $k=1,\cdots,K$, and using Lemma \ref{boundinglemma}, we get
\begin{align}
\begin{aligned}\notag
\E&\left[f(\btheta^{K+1})\right]-f(\btheta^1)\\
\leq&-\eta\sum_{k=1}^K\E\left[\lambda\<\nabla f(\btheta^k),\btheta^k\>+\left\|\nabla f(\btheta^k)\right\|_1\right]+2\eta\sqrt{d}\sum_{k=1}^K\E\left[\|\bdelta^k\|\right]+2KdL\eta^2\\
\leq&-\eta\sum_{k=1}^K\E\left[\lambda\<\nabla f(\btheta^k),\btheta^k\>+\left\|\nabla f(\btheta^k)\right\|_1\right]\\&+2\eta\sqrt{d}\left(\frac{\sigma}{1-\beta_2} + \frac{2KL\eta\sqrt{d}}{1-\beta_2} + (|\beta_1-\beta_2|+1-\beta_1) \cdot\frac{K\sigma}{\sqrt{1-\beta_2}}\right)+2KdL\eta^2.
\end{aligned}
\end{align}
Letting $\beta_1=1-\frac{c_1}{\sqrt{K}}$, $\beta_2=1-\frac{c_2}{\sqrt{K}}$, and $\eta=\frac{c_3}{\sqrt{d}K^{3/4}}$, we have
\begin{align}
\begin{aligned}\label{equ2}
\E&\left[f(\btheta^{K+1})\right]-f^\star + \eta\sum_{k=1}^K\E\left[\lambda\<\nabla f(\btheta^k),\btheta^k\>+\left\|\nabla f(\btheta^k)\right\|_1\right]\\
\leq& f(\btheta^1)-f^\star+\\&2\eta\sqrt{d}\left(\frac{\sigma}{1-\beta_2} + \frac{2KL\eta\sqrt{d}}{1-\beta_2} + (|\beta_1-\beta_2|+1-\beta_1) \cdot\frac{K \sigma}{\sqrt{1-\beta_2}}\right)+2KdL\eta^2\\
=& f(\btheta^1)-f^\star+\frac{2c_3\sigma}{c_2K^{1/4}} + \frac{4Lc_3^2}{c_2} + (|c_1-c_2|+c_1)\cdot\frac{2c_3\sigma}{\sqrt{c_2}}+\frac{2Lc_3^2}{K^{1/2}}\\
\leq&f(\btheta^1)-f^\star+\frac{2c_3\sigma}{c_2K^{1/4}} + \frac{4Lc_3^2}{c_2} + \frac{2c_3\sigma(2c_1+c_2)}{\sqrt{c_2}}+\frac{2Lc_3^2}{K^{1/2}},
\end{aligned}
\end{align}
and
\begin{align}
\begin{aligned}\notag
\frac{1}{K}\sum_{k=1}^K&\E\left[\lambda\<\nabla f(\btheta^k),\btheta^k\>+\left\|\nabla f(\btheta^k)\right\|_1\right]\\
\leq&\sqrt{d}\frac{f(\btheta^1)-f^\star}{c_3K^{1/4}}+\frac{2\sigma\sqrt{d}}{c_2K^{1/2}} + \frac{4Lc_3\sqrt{d}}{c_2K^{1/4}}+ \frac{2\sigma(2c_1+c_2)\sqrt{d}}{\sqrt{c_2}K^{1/4}}+\frac{2Lc_3\sqrt{d}}{K^{3/4}},
\end{aligned}
\end{align}
which completes the proof.
\end{proof}
\subsection{Proof of Lemma \ref{lemma:finite-horizon-bound}}
\begin{proof}
Borrowed from \cite{xie2024implicit}, for $k=1,\ldots,K-1$, we have
\begin{align}
\begin{aligned}\notag
\|\btheta^{k+1}\|_{\infty}-\frac{1}{\lambda}
&\leq
(1-\eta\lambda)
\left(
\|\btheta^k\|_{\infty}-\frac{1}{\lambda}
\right)\\
&\leq
(1-\eta\lambda)^k
\left(
\|\btheta^1\|_{\infty}-\frac{1}{\lambda}
\right)\\
&\leq
-\frac{1}{2\lambda}(1-\eta\lambda)^k.
\end{aligned}
\end{align}
Since $\eta\lambda\leq \frac{1}{2K}$, Bernoulli's inequality implies that, for every $k=1,\ldots,K-1$,
\begin{align}
\begin{aligned}\notag
(1-\eta\lambda)^k
\geq
1-k\eta\lambda
\geq
1-K\eta\lambda
\geq
\frac{1}{2}.
\end{aligned}
\end{align}
Therefore,
\begin{align}
\begin{aligned}\notag
\|\btheta^{k+1}\|_{\infty}
\leq
\frac{1}{\lambda}
-
\frac{1}{4\lambda}
=
\frac{3}{4\lambda},
\end{aligned}
\end{align}
for every $k=1,\ldots,K-1$. Together with
$\|\btheta^1\|_{\infty}\leq \frac{1}{2\lambda}\leq \frac{3}{4\lambda}$,
we obtain
\begin{align}
\begin{aligned}\notag
\lambda\|\btheta^k\|_{\infty}\leq \frac{3}{4},
\qquad
k=1,2,\ldots,K,
\end{aligned}
\end{align}
which completes the proof.
\end{proof}
\subsection{Proof of Theorem \ref{theorem2}}
\begin{proof}
As stated in Section \ref{sec:unccase}, since all assumptions and conditions of Theorem \ref{constrainedtheorem} are satisfied, all the conclusions and intermediate inequalities in its proof are applicable here, although we are studying the unconstrained case.
By Lemma \ref{lemma:finite-horizon-bound}, we have 
\begin{align}
\begin{aligned}\notag
\lambda\<\nabla f(\btheta^k),\btheta^k\>\geq -\lambda\left\|\nabla f(\btheta^k)\right\|_1\|\btheta^k\|_{\infty}\geq -\frac{3}{4}\left\|\nabla f(\btheta^k)\right\|_1.
\end{aligned}
\end{align}
Applying Theorem \ref{constrainedtheorem}, we get
\begin{align}
\begin{aligned}\notag
&\frac{1}{4K}\sum_{k=1}^K\E\left[\left\|\nabla f(\btheta^k)\right\|_1\right]\\
&= \frac{1}{K}\sum_{k=1}^K\E\left[-\frac{3}{4}\left\|\nabla f(\btheta^k)\right\|_1+\left\|\nabla f(\btheta^k)\right\|_1\right]\\
&\leq \frac{1}{K}\sum_{k=1}^K\E\left[\lambda\<\nabla f(\btheta^k),\btheta^k\>+\left\|\nabla f(\btheta^k)\right\|_1\right]\\
&\leq\sqrt{d}\frac{f(\btheta^1)-f^\star}{c_3K^{1/4}}+\frac{2\sigma\sqrt{d}}{c_2K^{1/2}} + \frac{4Lc_3\sqrt{d}}{c_2K^{1/4}}+\frac{2\sigma(2c_1+c_2)\sqrt{d}}{\sqrt{c_2}K^{1/4}}+\frac{2Lc_3\sqrt{d}}{K^{3/4}},
\end{aligned}
\end{align}
which completes the proof.
\end{proof}

\section{Additional Experiments}\label{sec:ae}
\subsection{Vision Tasks: Evaluating on CIFAR-10 and ImageNet-1K Datasets}\label{subsec:vt}
As is stated in Section \ref{sec:exp}, we perform additional experiments on CIFAR-10 \cite{krizhevsky2009learning} and ImageNet-1K \cite{ILSVRC15} datasets, with results reported in Figures \ref{fig:cifar10} and \ref{fig:imagenet}, respectively. These results further corroborate our findings.
\begin{figure}[tbhp]
    \centering
    \includegraphics[width=\textwidth]{legend_vision.pdf}

    \begin{subfigure}{0.5\textwidth}
        \centering
        \includegraphics[width=\linewidth]{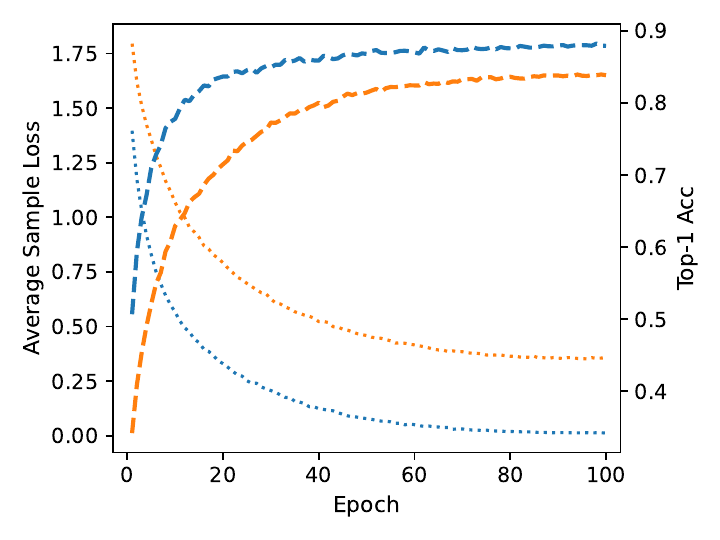}
        \caption{ResNet18 training loss and Top-1 accuracy}
    \end{subfigure}%
    \begin{subfigure}{0.5\textwidth}
        \centering
        \includegraphics[width=\linewidth]{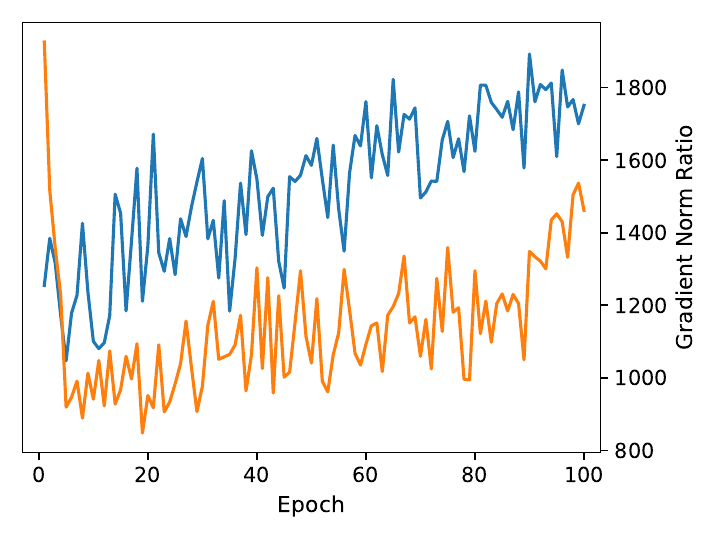}
        \caption{ResNet18 gradient norm ratio ($\sqrt{d}\approx3343$)}
    \end{subfigure}

    \begin{subfigure}{0.5\textwidth}
        \centering
        \includegraphics[width=\linewidth]{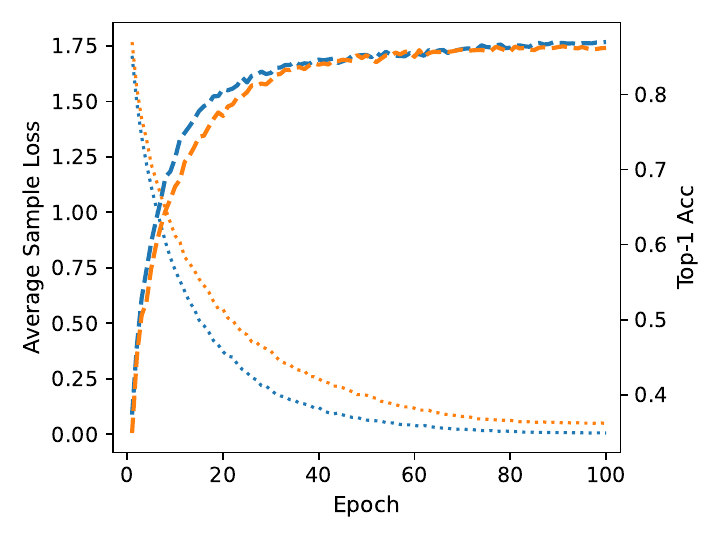}
        \caption{ResNet50 training loss and Top-1 accuracy}
    \end{subfigure}%
    \begin{subfigure}{0.5\textwidth}
        \centering
        \includegraphics[width=\linewidth]{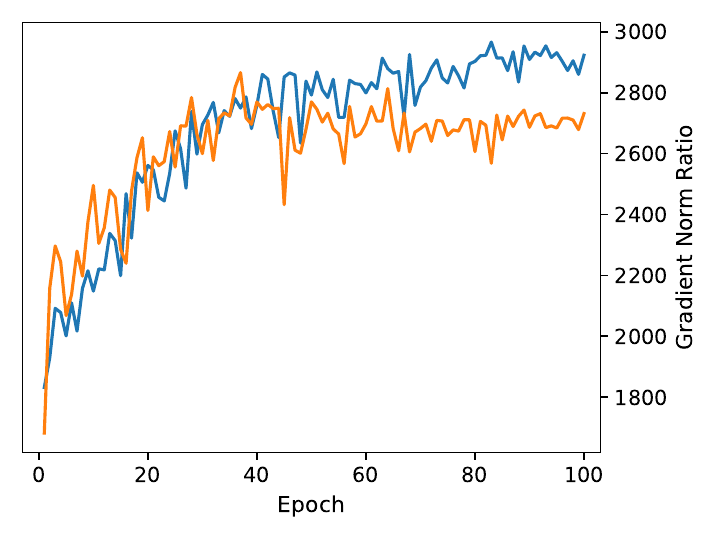}
        \caption{ResNet50 gradient norm ratio ($\sqrt{d}\approx 4850$)}
    \end{subfigure}

    \begin{subfigure}{0.5\textwidth}
        \centering
        \includegraphics[width=\linewidth]{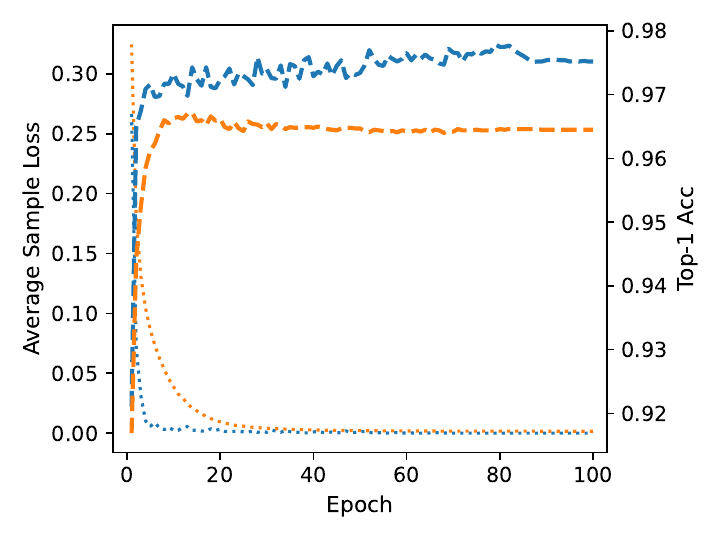}
        \caption{ViT-S training loss and Top-1 accuracy}
    \end{subfigure}%
    \begin{subfigure}{0.5\textwidth}
        \centering
        \includegraphics[width=\linewidth]{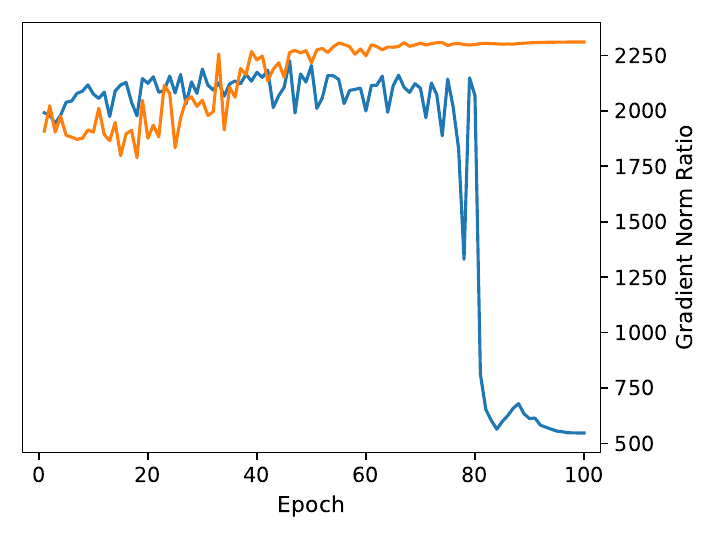}
        \caption{ViT-S gradient norm ratio ($\sqrt{d}\approx 4655$)}
    \end{subfigure}
    \caption{Overview of results of ResNet18 \cite{he2016deep}, ResNet50 \cite{he2016deep}, and ViT-S \cite{dosovitskiy2020image} models training and evaluating on CIFAR-10 \cite{krizhevsky2009learning} dataset.}
    \label{fig:cifar10}
\end{figure}
\begin{figure}[tbhp]
    \centering
    \includegraphics[width=\textwidth]{legend_vision.pdf}

    \begin{subfigure}{0.5\textwidth}
        \centering
        \includegraphics[width=\linewidth]{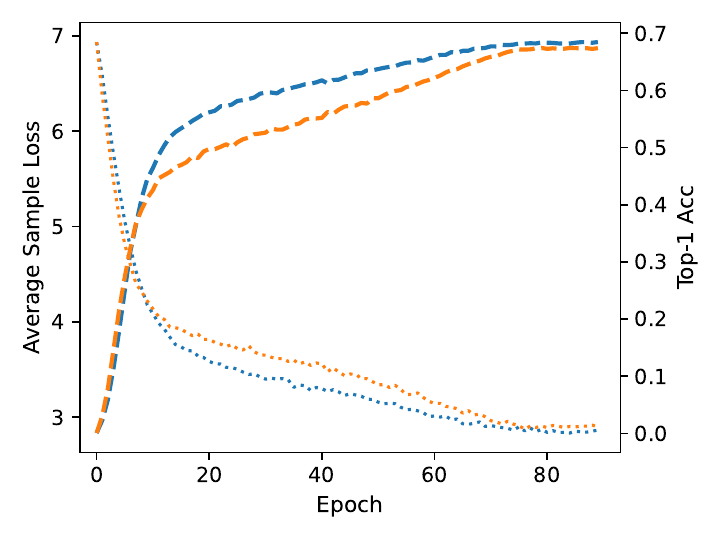}
        \caption{ResNet18 training loss and Top-1 accuracy}
    \end{subfigure}%
    \begin{subfigure}{0.5\textwidth}
        \centering
        \includegraphics[width=\linewidth]{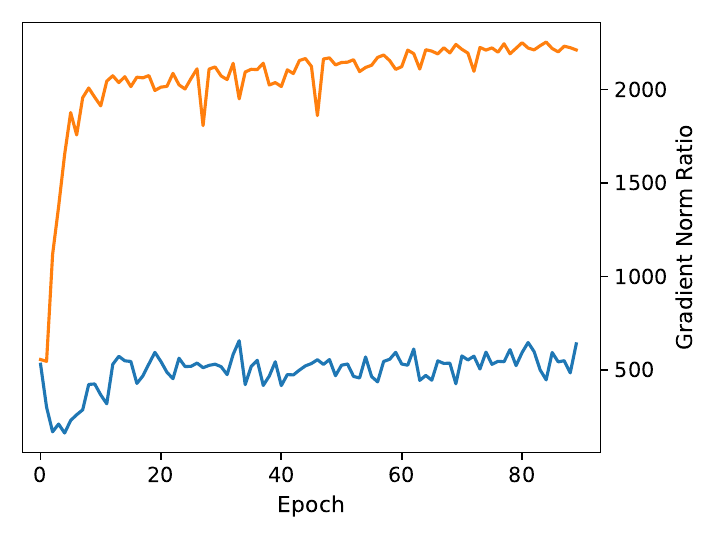}
        \caption{ResNet18 gradient norm ratio ($\sqrt{d}\approx3350$)}
    \end{subfigure}

    \begin{subfigure}{0.5\textwidth}
        \centering
        \includegraphics[width=\linewidth]{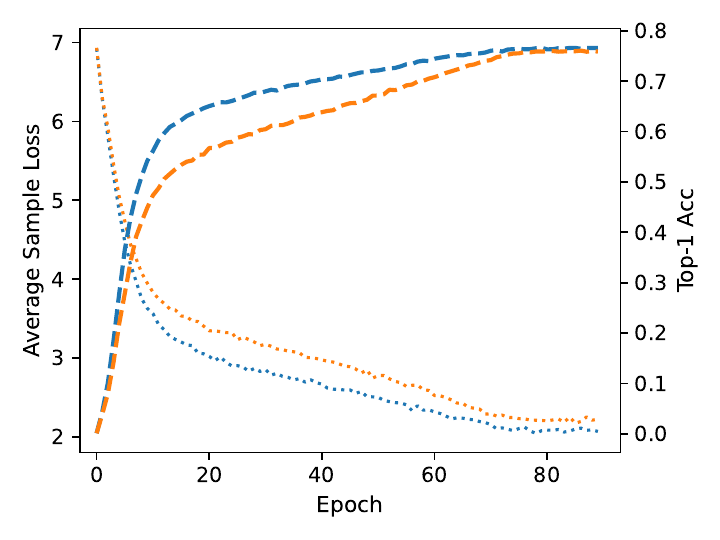}
        \caption{ResNet50 training loss and Top-1 accuracy}
    \end{subfigure}%
    \begin{subfigure}{0.5\textwidth}
        \centering
        \includegraphics[width=\linewidth]{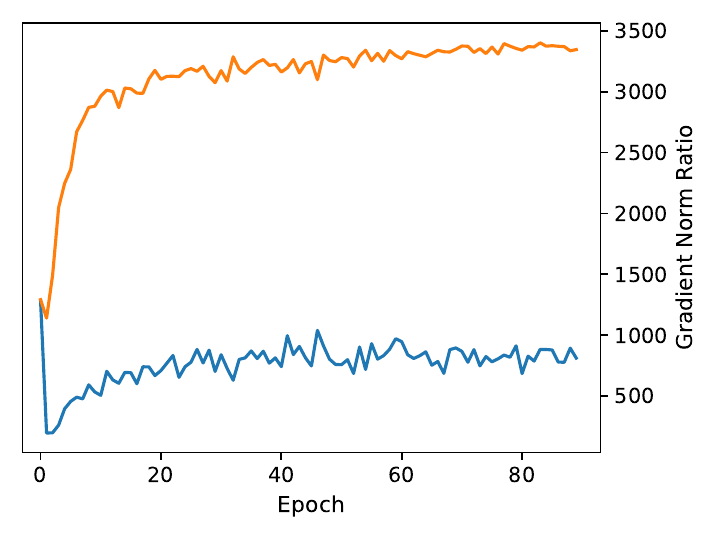}
        \caption{ResNet50 gradient norm ratio ($\sqrt{d}\approx 4869$)}
    \end{subfigure}

    \begin{subfigure}{0.5\textwidth}
        \centering
        \includegraphics[width=\linewidth]{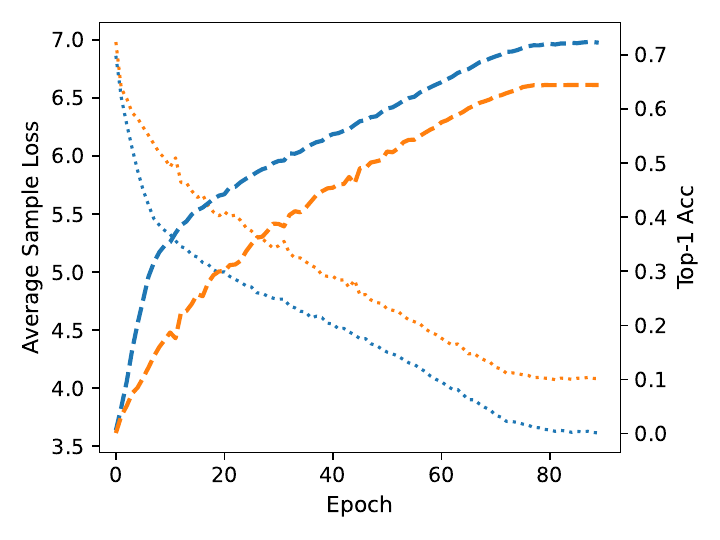}
        \caption{ViT-S training loss and Top-1 accuracy}
    \end{subfigure}%
    \begin{subfigure}{0.5\textwidth}
        \centering
        \includegraphics[width=\linewidth]{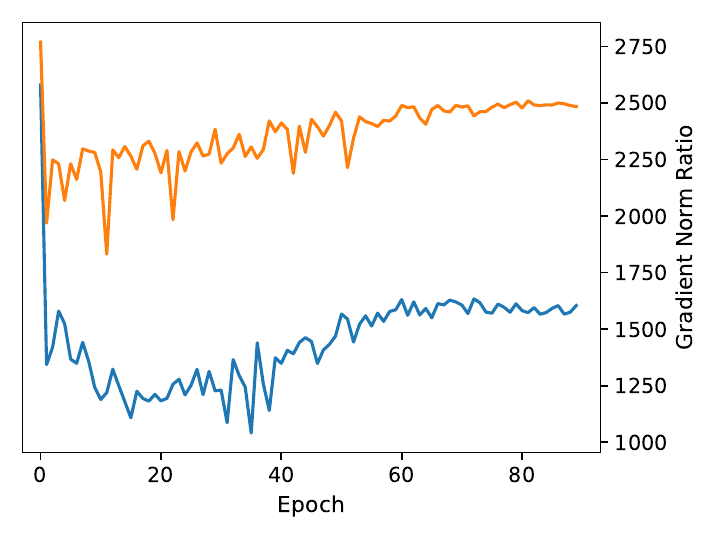}
        \caption{ViT-S gradient norm ratio ($\sqrt{d}\approx 5484$)}
    \end{subfigure}
    \caption{Overview of results of ResNet18 \cite{he2016deep}, ResNet50 \cite{he2016deep}, and ViT-S \cite{dosovitskiy2020image} models training and evaluating on ImageNet-1K \cite{ILSVRC15} dataset.}
    \label{fig:imagenet}
\end{figure}
\subsection{Language Tasks: Assessing the Zero-shot Performance on WikiText-103 Dataset}\label{subsec:lt}
In this task, we assess the zero-shot performance of GPT-2 models by extracting checkpoints at 10k, 20k, 30k, 40k, and 50k training steps. We use these intermediate weights to compute perplexity on the WikiText-103 dataset \cite{merity2016pointer} as a measure of generalization performance without any task-specific finetuning.

The results are shown in Table \ref{tab:performance-metrics}. To ensure numerical consistency in scale, we follow the common practice of taking the logarithm of perplexity values. Notably, LION surpasses SGD's final performance from the outset, achieving lower perplexity values even at the earliest training stage.

\begin{table}[tbhp]
    \centering
    \caption{Log perplexity values on WikiText-103 dataset for GPT-2 Small and Medium models training with SGD and LION optimizers, at 10k, 20k, 30k, 40k, and 50k training steps. LION outperforms SGD's final result from the beginning, demonstrating faster convergence to better performance.}
    \label{tab:performance-metrics}
    \begin{tabular}{ccccccc}
        \toprule
        \multirow{2}{*}{\textbf{Model}} & \multirow{2}{*}{\textbf{Optimizer}} & \multicolumn{5}{c}{\textbf{Steps}} \\
        & & 10k & 20k & 30k & 40k & 50k \\
        \midrule
        \multirow{2}{*}{\textbf{GPT-2 Small}} & SGD & 8.202 & 8.005 & 7.922 & 7.892 & 7.887 \\
                                     & LION & \textbf{3.659} & \textbf{3.409} & \textbf{3.326} & \textbf{3.293} & \textbf{3.284} \\
        \midrule
        \multirow{2}{*}{\textbf{GPT-2 Medium}} & SGD & 8.029 & 7.807 & 7.737 & 7.715 & 7.712 \\
                                      & LION & \textbf{3.221} & \textbf{3.004} & \textbf{2.918} & \textbf{2.882} & \textbf{2.876} \\
        \bottomrule
    \end{tabular}
\end{table}
\section{Detailed Training Settings}\label{sec:dts}
\textbf{CIFAR-10 and CIFAR-100 Training}: Given that optimal learning rates may differ across datasets and models, we perform a grid search over the learning rates set to [3e-3, 1e-3, 3e-4, 1e-4, 3e-5, 1e-5, 3e-6, 1e-6]. This approach is intended to span a broad range of typical learning rate values and identify the most suitable rate for each scenario. We restrict our training to 100 epochs and employ the classic cosine annealing learning rate schedule approach. We keep the batch size to 64. To prevent overfitting, we set weight decay to 0.1 globally. Given that Vision Transformers are originally designed for handling images of at least 224$\times$224 pixels, a common method is to finetune an ImageNet pretrained model on the CIFAR datasets. We employ this strategy and demonstrate the superior performance of the LION optimizer in such finetuning tasks. These experiments are conducted on a single NVIDIA A6000 GPU.

\textbf{ImageNet Training}: We adhere to the standard training configuration commonly used for numerous well-established deep learning optimizers \cite{he2016deep}. It contains 90 epochs in total, initiating with a 10-epoch warmup phase for the learning rate, followed by the remaining epochs controlled by a cosine annealing scheduler. For SGD, we set the learning rate to 0.5 to align with the vanilla setting. For LION, we still perform a grid search on the same candidate list. The batch size is set to 2048. These experiments are conducted on 8 NVIDIA A6000 GPUs.

\textbf{BERT and GPT-2 Training:} We employ the prevalent Megatron-LM \cite{shoeybi2019megatron} framework developed by NVIDIA to train and evaluate models on the OpenWebText \cite{Gokaslan2019OpenWeb} Dataset. We make only minimal modifications to the original training configurations by setting the SGD learning rate to 1e-4, the LION learning rate to 1e-5, and the global batch size to 640, then apply this setting universally to all language tasks. We accumulate the training loss and gradient information per 100 batches, covering a total of 64000 samples. The whole training process consists of 50000 steps and is conducted on 16 NVIDIA A100 GPUs.
\bibliographystyle{elsarticle-num} 
\bibliography{mybibfile}
\end{document}